%% file: ICLR2020.tex
\documentclass{article} 
\usepackage{iclr2020_conference,times}

\input{math_commands.tex}

\usepackage{algorithm}
\usepackage{algorithmic}

\usepackage[pdftex]{graphicx}

\usepackage{amsthm}
\usepackage{amsmath}
\usepackage{booktabs} 
\usepackage{hyperref}
\usepackage{url}
\usepackage{tikz}

\newtheorem{prop}{Proposition}
\newtheorem{lemma}{Lemma}
\newtheorem{theo}{Theorem}

\def\diag{\text{diag}}

\usepackage{array}
\newcolumntype{M}[1]{>{\centering\arraybackslash}m{#1}}

\title{Spectral  Embedding of Regularized Block Models}

\author{Nathan De Lara \& Thomas Bonald\\
Institut Polytechnique de Paris\\
Paris, France \\
\texttt{\{nathan.delara, thomas.bonald\}@telecom-paris.fr} 
}

\iclrfinalcopy 
\begin{document}

\maketitle

\begin{abstract}
Spectral embedding is a popular technique for the representation of graph data. Several regularization techniques have been proposed to improve the quality of the embedding with respect to downstream tasks like clustering. In this paper, we explain on a simple block model the impact of the complete graph regularization, whereby a constant is added to all entries of the adjacency matrix. Specifically, we show that the regularization forces the spectral embedding  to focus on  the  largest blocks, making the representation less sensitive to noise or outliers. We illustrate these results on both  on both synthetic and real data, showing how regularization improves standard clustering scores. 
\end{abstract}

\section{Introduction}
\label{intro}

Spectral embedding  is a standard technique for the representation of graph data \citep{ng2002spectral, belkin2002laplacian}. Given the adjacency matrix $A \in \mathbb{R}_+^{n \times n}$ of the graph, it is obtained by solving either   the  eigenvalue problem:
\begin{equation}
    \label{eq:spectral_laplacian}
    LX = X\Lambda, \text{ with } X^TX = I,
\end{equation}
or  the generalized eigenvalue problem:
\begin{equation}
    \label{eq:spectral}
    LX = DX\Lambda, \text{ with } X^TDX = I,
\end{equation}
where $D=\diag(A1_n)$ is the degree matrix, with $1_n$ the all-ones vector of dimension $n$, $L = D-A$ is the Laplacian matrix of the graph,  $\Lambda \in \mathbb{R}^{k \times k}$ is the diagonal matrix of the $k$ smallest (generalized) eigenvalues of $L$  and  $X \in \mathbb{R}^{n \times k}$ is the corresponding matrix of (generalized) eigenvectors. 
In this paper, we only consider the generalized eigenvalue problem, whose solution is given by the spectral decomposition of the normalized Laplacian matrix $L_{\rm norm}= I- D^{-1/2}AD^{-1/2}$ \citep{luxburg07}.

The spectral embedding can be interpreted as  equilibrium states of some physical systems \citep{snell00,spielman2007spectral,bonald2018weighted}, a desirable property in modern machine learning. However, it tends to produce poor results on real datasets if applied directly on the graph \citep{amini2013pseudo}. One reason is that real graphs are most often disconnected due to noise or outliers in the dataset.

In order to improve the quality of the  embedding, two main types of regularization have been proposed. The first artificially increases the degree of each node by a constant factor \citep{chaudhuri2012spectral, qin2013regularized}, while the second adds a constant to all entries of the original adjacency matrix \citep{amini2013pseudo, joseph2016impact, zhang2018understanding}.
In the practically interesting case where the original adjacency matrix $A$ is sparse, the regularized adjacency matrix is dense but has a so-called sparse $+$ low rank structure, enabling the  computation of  the spectral embedding on  very large graphs  \citep{mlg2019_1}.

While \citep{zhang2018understanding} explains the effects of regularization through graph conductance and \citep{joseph2016impact} through eigenvector perturbation on the Stochastic Block Model, there is no simple interpretation of the benefits of graph regularization. In this paper, we show on a simple block model that  the complete graph regularization forces the spectral embedding to separate the blocks in decreasing order of size, making the embedding less sensitive to noise or outliers in the data. 

Indeed, \citep{zhang2018understanding} identified that, without regularization, the cuts corresponding to the first dimensions of the spectral embedding tend to separate small sets of nodes, so-called dangling sets, loosely connected to the rest of the graph. Our work shows more explicitly    that regularization forces  the spectral embedding to focus on the largest clusters. Moreover, our analysis involves some explicit characterization of the eigenvalues, allowing us to quantify   the impact of the regularization parameter.

The rest of this paper is organized as follows. Section \ref{sec:bmls} presents block models and an important preliminary result about their aggregation. Section \ref{sec:pb_stmt} presents the main result of the paper, about the regularization of block models, while Section \ref{sec:bipartite}  extends this result to bipartite graphs.  Section \ref{sec:exp} presents the experiments and Section \ref{sec:conc} concludes the paper.

\section{Aggregation of Block Models}
\label{sec:bmls}

Let $A \in \mathbb{R}_+^{n \times n}$ be the adjacency matrix of an undirected, weight graph, that is a symmetric matrix such that $A_{ij}>0$ if and only if there is an edge between nodes $i$ and $j$, with weight $A_{ij}$.  Assume that the $n$ nodes of the graph can be partitioned into $K$ blocks of respective sizes 
$n_1,\ldots,n_K$
so that 
any two nodes of the same block have the same neighborhood, i.e., the corresponding rows (or columns) of $A$ are the same. 
Without any loss of generality, we assume that the matrix $A$ has rank $K$.
We refer to such a graph as a block model.

Let $Z \in \mathbb{R}^{n \times K}$ be the associated membership matrix, with $Z_{ij} = 1$ if index $i$ belongs to block $j$ and $0$ otherwise. 
We denote by $W = Z^TZ \in \mathbb{R}^{K \times K}$ the diagonal matrix of block sizes.

Now define $\Bar{A} = Z^TAZ \in \mathbb{R}^{K \times K}$. This is the adjacency matrix of the aggregate graph, where each block of the initial graph is replaced by a single node; two nodes in this graph are connected by  an edge of weight equal to the total weight of  edges between the corresponding blocks in the original graph.
We denote by  $\Bar{D}=\diag(\bar A 1_K)$  the  degree matrix and by $\bar L = \bar D - \bar A $ the Laplacian matrix of the aggregate graph.

The following result shows that the solution to the generalized eigenvalue problem (\ref{eq:spectral}) follows from that of the aggregate graph:

\begin{prop}
\label{prop:aggregate}
Let $x$ be a solution to the generalized eigenvalue problem:
\begin{equation}
    \label{eq:spectral_original}
     Lx = \lambda  D x.
\end{equation}
Then either $Z^T x = 0$ and $\lambda = 1$ or $x = Z y$ where 
 $y$ is a solution to the generalized eigenvalue problem:
\begin{equation}
    \label{eq:spectral_aggregate}
    \bar Ly = \lambda \bar Dy.
\end{equation}
\end{prop}

\begin{proof}

Consider the following reformulation of the generalized eigenvalue problem (\ref{eq:spectral_original}):
\begin{equation}
    \label{eq:genls}
    Ax = Dx(1 - \lambda).
\end{equation}

Since the rank of $A$  is equal to $K$,  there are $n-K$  eigenvectors $x$ associated with the eigenvalue $\lambda = 1$,   each satisfying $Z^Tx = 0$. By orthogonality, the other eigenvectors satisfy $x = Zy$ for some vector $y\in \mathbb{R}^K$. We get:
$$
AZy = DZ y (1-\lambda),
$$
so that
$$
\bar A y = \bar D y(1-\lambda).
$$
Thus $y$ is a solution to the generalized eigenvalue problem (\ref{eq:spectral_aggregate}).
\end{proof}



\section{Regularization of Block Models}
\label{sec:pb_stmt}

Let $A$ be the adjacency matrix of some undirected graph. 
We consider a regularized version of the graph where an edge of weight $\alpha$ is added between all pairs of nodes, for some constant $\alpha >  0$. The corresponding adjacency matrix is given by:
\begin{equation*}
    \label{eq:reg}
    A_\alpha = A + \alpha J,
\end{equation*}
where $J=1_n1_n^T$ is the all-ones matrix  of same dimension as $A$. 
We  denote by $D_\alpha=\diag(A_\alpha 1_n)$ the corresponding degree matrix and by $L_\alpha = D_\alpha - A_\alpha$ the Laplacian matrix.

We first consider a simple block model where the graph consists  of $K$ disjoint cliques of respective sizes  $n_1 > n_2 > \dots > n_K$ nodes, with $n_K \ge 1$. In this case, we have $A = ZZ^T$, where $Z$ is the membership matrix.

The objective of this section is to demonstrate that, in this setting, the $k$-th dimension of the spectral embedding isolates the $k-1$ largest cliques from the rest of the graph, for any $k\in \{2,\ldots,K\}$

\begin{lemma}
\label{lem:simple_order}
Let $\lambda_1\le  \lambda_2\le \ldots\le  \lambda_n$ be the eigenvalues associated with  the  generalized eigenvalue problem:
\begin{equation}\label{eq:eigreg}
    L_\alpha x = \lambda D_\alpha x.
\end{equation}
We have $\lambda_1 = 0 < \lambda_2 \le \ldots \le \lambda_K < \lambda_{K+1} =\ldots=\lambda_n = 1$.
\end{lemma}
\begin{proof}
Since the Laplacian matrix $L_\alpha$ is positive semi-definite, all eigenvalues are non-negative \citep{chung}. 
We know that the eigenvalue 0 has multiplicity 1 on 
observing that the regularized graph is connected.
Now for any vector $x$,
$$
x^T A_\alpha x = x^T A x +\alpha x^T J x = ||Z^T x||^2 + \alpha (1_n^Tx)^2\ge 0,
$$
so that the matrix $A_\alpha$ is positive semi-definite. In view of (\ref{eq:genls}), this shows that  $\lambda \le 1$ for any eigenvalue $\lambda$. 
The proof then follows from 
 Proposition \ref{prop:aggregate}, on observing that the eigenvalue 1 has multiplicity $n-K$. 
\end{proof}

\begin{lemma}
\label{lem:eigenvector}
Let $x$ be a solution to the generalized eigenvalue problem (\ref{eq:eigreg}) with $\lambda \in (0,1)$.
 There exists some $s\in \{+1, -1\}$ such that for each node $i$ in block $j$,
$$
{\rm sign}(x_i) = s \quad \Longleftrightarrow \quad n_j \ge \alpha  \frac{1-\lambda}\lambda n.
$$
\end{lemma}
\begin{proof}
In view of Proposition \ref{prop:aggregate}, we have $x=Zy$ where $y$ is a solution to the  generalized eigenvalue problem of the aggregate graph, with adjacency matrix:
$$    \Bar{A}_\alpha = Z^T A_\alpha Z = Z^T(A + \alpha J)Z.
    $$
   
    Since $A = ZZ^T$ and $W   = Z^TZ$, we have
    $
\Bar{A}_\alpha = W^2 + \alpha Z^T J Z.
$
Using the fact that $Z1_K = 1_n$, we get $J = 1_n1_n^T = ZJ_K Z^T$ with $J_K = 1_K1_K^T$ the all-ones matrix of dimension $K\times K$, so that:
$$
\Bar{A}_\alpha = W(I_K + \alpha J_K)W,
$$
 where $I_K$ is the identity matrix  of dimension $K\times K$.  We deduce the degree matrix:
\begin{align*}
\Bar{D}_\alpha
= W( W +  \alpha n I_K),
\end{align*}
and 
the Laplacian matrix:
$$
\Bar{L}_\alpha = \bar D_\alpha - \bar A_{\alpha} = \alpha W(n   I_K -  J_K W).
$$

The generalized eigenvalue problem associated with the aggregate graph
is:
$$
\Bar{L}_\alpha y = \lambda \bar D_\alpha y.
$$
After multiplication by $W^{-1}$, we get:
$$
\alpha (nI_K -J_K W)   y = \lambda (W + \alpha n I_K)y.
$$
Observing that $J_KW y = 1_K1_K^TWy = (1_K^TWy)1_K\propto 1_K$, we conclude that:
\begin{equation}\label{eq:eigenvector0}
(\alpha n(1-\lambda) - \lambda W) y\propto 1_K,
\end{equation}
and since $W=\diag(n_1,\ldots,n_K)$,
\begin{equation}\label{eq:eigenvector}
\forall j = 1,\ldots,K,\quad y_j \propto  \frac 1 {\lambda n_j -   \alpha(1-\lambda)n}.
\end{equation}
The result then follows from the fact that $x = Z y$.
\end{proof}

\begin{lemma}
\label{lem:order_eig}
The $K$ smallest eigenvalues satisfy:
\begin{align*}
    0 = \lambda_1 < \mu_1 < \lambda_2 < \mu_2 < \dots < \lambda_{K} < \mu_K,
\end{align*}
where for all $j=1,\ldots,K,$
$$
 \mu_j = \frac{\alpha n} {\alpha n +n_j}.
$$
\end{lemma}

\begin{proof}
We know from Lemma \ref{lem:simple_order} that the $K$ smallest eigenvalues are in $[0,1)$.
Let $x$ be a solution to the generalized eigenvalue problem (\ref{eq:eigreg}) with $\lambda \in (0,1)$.
We know that $x=Zy$ where $y$ is an eigenvector associated with the same eigenvalue $\lambda$ for the aggregate graph. Since $1_K$ is an eigenvector for the eigenvalue 0, we have  $y^T \bar D_\alpha 1_K = 0$. Using the fact that $\bar D_\alpha = W(W+\alpha n I_K)$, we get:
\begin{align*}
  \sum_{j=1}^Kn_j(n_j+\alpha n) y_j = 0.
  \end{align*}
  We then deduce from  (\ref{eq:eigenvector0}) and (\ref{eq:eigenvector}) that $\lambda\not\in\{\mu_1,\ldots,\mu_K\}$ and
  $$
  \sum_{j=1}^Kn_j(n_j+\alpha n) \dfrac{1}{\lambda /\mu_j -1} = 0.
$$
This condition cannot be satisfied if $\lambda< \mu_1$ or $\lambda > \mu_K$ as  the terms of the sum would be either all positive or all negative. 

Now let $y'$ be another eigenvector for the aggregate graph, with $y^T \bar D_\alpha y' = 0$, for the eigenvalue $\lambda'\in (0,1)$. By the same argument, we get:
\begin{align*}
  \sum_{j=1}^Kn_j(n_j+\alpha n) y_jy'_j = 0,
  \end{align*}
  and
  $$
  \sum_{j=1}^Kn_j(n_j+\alpha n) \dfrac{1}{\lambda /\mu_j -1} \frac{1}{\lambda' /\mu_j -1} = 0.
$$
with $\lambda'\not\in\{\mu_1,\ldots,\mu_K\}$.
This condition cannot be satisfied if $\lambda$ and $\lambda'$ are in the same interval $(\mu_j, \mu_{j+1})$ for some $j$ as  the terms in the sum would be all positive. There are $K-1$  eigenvalues in $(0,1)$ for $K-1$ such intervals, that is one eigenvalue per interval.
\end{proof}

The main result of the paper is the following, showing that the $k-1$ largest cliques of the original graph can be recovered from the spectral embedding of the regularized graph in dimension $k$.

\begin{theo}
\label{theo:main}
Let $X$ be the spectral embedding of dimension $k$, as defined by (\ref{eq:spectral}), for some $k$ in the set $\{2,\ldots,K\}$. Then $\text{sign}(X)$ gives the $k-1$ largest blocks of the graph.
\end{theo}
\begin{proof}
Let $x$ be the $j$-th column of the matrix $X$, for some $j\in \{2,\ldots,k\}$. In view of Lemma \ref{lem:order_eig}, this is the eigenvector associated with eigenvalue $\lambda_j \in (\mu_{j-1},\mu_j)$, so that
$$
\alpha \frac{1-\lambda_j}{\lambda_j}n \in (n_{j-1},n_j).
$$
In view of Lemma \ref{lem:eigenvector}, all entries of $x$ corresponding to blocks of size  $n_1,n_2\ldots,n_{j-1}$ have the same sign, the other having the opposite sign.
\end{proof}

Theorem \ref{theo:main} can be extended in several ways. First, the assumption of distinct block sizes can easily be  relaxed. If there are $L$ distinct values of block sizes, say $m_1,\ldots,m_L$ blocks of sizes $n_1> \ldots > n_L$, there are $L$ distinct values for the thresholds $\mu_j$ and thus $L$ distinct values for the eigenvalues $\lambda_j$ in $[0,1)$, the multiplicity of the $j$-th smallest eigenvalue being equal to $m_j$. The spectral embedding in dimension $k$ still gives  $k-1$  cliques of the largest sizes. 

Second, the  graph may have edges between blocks. Taking $A = ZZ^T +\varepsilon J$ for instance, for some parameter $\varepsilon \ge 0$,  the results are exactly the same, with $\alpha$ replaced by $\epsilon +\alpha$. A key observation is that regularization really matters when $\varepsilon \to 0$, in which case the initial graph becomes disconnected and, in the absence of regularization, the spectral embedding may isolate small connected components of the graph. In particular, the regularization makes the spectral embedding much less sensitive to noise, as will be demonstrated in the experiments.

Finally, degree correction can be added by varying the node degrees within blocks. Taking $A = \theta ZZ^T \theta$, for some arbitrary diagonal matrix $\theta$ with positive entries, similar results can be obtained under the regularization $A_\alpha = A + \alpha \theta J \theta$. Interestingly, the spectral embedding in dimension $k$ then recovers the $k-1$ largest blocks in terms of {\it normalized weight}, the ratio of the total weight of the block to the number of nodes in the block. 

\section{Regularization of  Bipartite Graphs}
\label{sec:bipartite}

Let $B= \mathbb{R}_+^{n\times m}$ be the biadjacency matrix of some bipartite graph with respectively $n,m$ nodes in each part, i.e., $B_{ij}>0$ if and only if there is an edge between node $i$ in the first part of the graph and node $j$ in the second part of the graph, with weight $B_{ij}$. 
 This is an undirected graph of $n+m$ nodes with adjacency matrix:
$$
A = \begin{bmatrix}
0 & B \\
B^T & 0
\end{bmatrix}
$$
The spectral embedding of the graph (\ref{eq:spectral}) can be written in terms of the biadjacency matrix as follows:
\begin{equation}\label{eq:gsvd}
\left\{
    \begin{array}{l}
    BX_2 = D_1 X_1 (I- \Lambda)\\
    B^TX_1 = D_2 X_2 (I- \Lambda)
    \end{array}\right.
\end{equation}
where $X_1,X_2$ are the embeddings of each part of the graph, with respective dimensions $n\times k$ and $m\times k$, $D_1 = \diag(B1_{m})$ and $D_2 = \diag(B^T1_{n})$. In particular, the spectral embedding of the graph follows from the generalized SVD of the biadjacency matrix $B$. 

The complete regularization adds edges between all pairs of nodes, breaking the bipartite structure of the graph. Another approach consists in applying the regularization to the biadjacency matrix, i.e., in considering the regularized bipartite graph with biadjacency matrix:
$$
B_\alpha = B +\alpha J,
$$
where $J = 1_n1_m^T$ is here the all-ones matrix of same dimension as $B$. The spectral embedding of the regularized graph is that associated with the adjacency matrix:
\begin{equation}
    \label{eq:biadj_reg}
A_\alpha = \begin{bmatrix}
0 & B_\alpha \\
B_\alpha^T & 0
\end{bmatrix}
\end{equation}

As in Section \ref{sec:pb_stmt}, we consider a block model so that the biadjacency matrix $B$ is block-diagonal with all-ones block matrices on the diagonal. Each part of the graph consists of $K$ groups of nodes of respective sizes $n_1>\ldots>n_K$ and $m_1>\ldots>m_K$, with nodes of block $j$ in the first part connected only to nodes of block $j$ in the second part, for all $j=1,\ldots,K$.

We consider the generalized eigenvalue problem (\ref{eq:eigreg}) associated with the above matrix $A_\alpha$. In view of (\ref{eq:gsvd}), this is equivalent to the generalized SVD of the regularized biadjacency matrix $B_\alpha$.
We have the following results, whose proofs are deferred to the appendix:

\begin{lemma}\label{lem:bip1}
Let $\lambda_1\le  \lambda_2\le \ldots \le \lambda_n$ be the eigenvalues associated with the generalized eigenvalue problem (\ref{eq:eigreg}).
We have $\lambda_1 = 0 < \lambda_2 \le \ldots \le \lambda_K < \lambda_{K+1} =\ldots=\lambda_{n-2K} < \ldots < \lambda_n = 2$.
\end{lemma}

\begin{lemma}\label{lem:bip2}
Let $x$ be a solution to the generalized eigenvalue problem (\ref{eq:eigreg}) with $\lambda \in (0,1)$.
There exists $s_1,s_2\in \{+1,-1\}$ such that for each node $i$ in block $j$ of part $p\in \{1,2\}$,
$$
{\rm sign}(x_i) = s_p \quad \Longleftrightarrow \quad  \dfrac{n_jm_j}{(n_j + \alpha n)(m_j + \alpha m)} \ge 1 - \lambda.
$$
\end{lemma}

\begin{lemma}\label{lem:bip3}
The $K$ smallest eigenvalues satisfy:
\begin{align*}
    0 = \lambda_1 < \mu_1 < \lambda_2 < \mu_2 < \dots < \lambda_{K} < \mu_K,
\end{align*}
where for all $j=1,\ldots,K,$
$$
 \mu_j = 1 - \dfrac{n_jm_j}{(n_j + \alpha n)(m_j + \alpha m)}.
$$
\end{lemma}

\begin{theo}
    \label{theo:svd}
Let $X$ be the spectral embedding of dimension $k$, as defined by (\ref{eq:spectral}), for some $k$ in the set $\{2,\ldots,K\}$. Then $\text{sign}(X)$ gives the $k-1$ largest blocks of each part of the graph.
\end{theo}

Like Theorem \ref{theo:main}, the assumption of decreasing block sizes can easily be  relaxed. Assume that block pairs are indexed in decreasing order of $\mu_j$. Then the spectral embedding of dimension $k$ gives the $k-1$ first block pairs for that order. It is interesting to notice that the order now depends on $\alpha$: when $\alpha\to 0^+$, the block pairs $j$ of highest value $(\frac n{n_j}  + \frac m {m_j})^{-1}$ (equivalently, highest {\it harmonic mean} of proportions of nodes in each part of the graph)  are isolated first; when $\alpha \to +\infty$, the block pairs $j$ of highest value $\frac {n_jm_j}{nm}$ (equivalently, the highest {\it geometric mean} of proportions of nodes in each part of the graph) are isolated first.

The results also extend to non-block diagonal biadjacency matrices $B$ and degree-corrected models, as for Theorem \ref{theo:main}.

\section{Experiments}
\label{sec:exp}

We now illustrate the impact of regularization on the quality of spectral embedding. We focus on a clustering task, using both synthetic and real datasets where the ground-truth clusters are known.  
In all experiments, we skip the first dimension of the spectral embedding as it is not informative (the corresponding eigenvector is the all-ones vector, up to some multiplicative constant). 
The code to reproduce these experiments is available online\footnote{\url{https://github.com/research-submissions/iclr20}}.

\subsection{Toy graph}
\label{ssec:toy}

We first illustrate the theoretical results of the paper with a toy graph consisting of 3 cliques of respective sizes $5,3,2$.
We compute the spectral embeddings in dimension 1, using the second smallest eigenvalue.   Denoting by $Z$ the membership matrix, we get  $X \approx Z (-0.08,0.11, 0.05)^T$ 
for $\alpha = 1$, showing that the embedding isolates the largest cluster; this is not the case in the absence of regularization, where $X \approx Z (0.1,-0.1, 0.41)^T$.

\subsection{Datasets}

This section describes the datasets used in our experiments. All  graphs are considered as undirected. Table \ref{tab:datasets} presents the main features of the graphs. 

\paragraph{Stochastic Block-Model (SBM)} We generate 100 instances of the same stochastic block model \citep{holland1983stochastic}. 
There are 100 blocks of size 20, with intra-block edge probability  set to $0.5$ for the first 50 blocks and $0.05$ for the other blocks. The inter-block edge probability is set to $0.001$
Other sets of parameters can be tested using the code available online. The ground-truth cluster of each node corresponds to its block.

\paragraph{20newsgroup (NG)} This dataset consists of around $18 000$ newsgroups posts on 20 topics. This defines a weighted bipartite graph between documents and words. The label of each document corresponds to the topic.

\paragraph{Wikipedia for Schools (WS)} \citep{haruechaiyasak2008article}. This is the  graph of hyperlinks between a subset of Wikipedia pages. The label of each page is its  category (e.g., countries, mammals, physics).

\begin{table}[h]
    \centering
    \caption{Main features of  the graphs.}
    \vspace{2mm}
    
    \begin{tabular}{l|ccc}
        dataset & SBM & NG & WS \\
        \hline
        \# nodes ($n$) & 2000 
        & 10723 & 4591 \\
        \# edges & $\approx 5.10^3$ 
        & $\approx 2.10^6$ & $\approx 2.10^5$ \\
        \# clusters in ground truth & 100 & 20 & 14 \\
    \end{tabular}
    \label{tab:datasets}
\end{table}

\subsection{Metrics}
\label{ssec:metrics}
We consider a large set of metrics from the clustering literature. All metrics are upper-bounded by 1 and the higher the score the better.

\paragraph{Homogeneity (H), Completeness (C) and V-measure score (V)} \citep{rosenberg2007v}. Supervised metrics. A cluster is homogeneous if all its data points are members of a single class in the ground truth. A clustering is complete if all the members of a class in the ground truth belong to the same cluster in the prediction. Harmonic mean of homogeneity and completeness.

\paragraph{Adjusted Rand Index (ARI)} \citep{hubert1985comparing}. Supervised metric. This is the corrected for chance version of the Rand Index which is itself an accuracy on pairs of samples.

\paragraph{Adjusted Mutual Information (AMI)} \citep{vinh2010information} Supervised metric. Adjusted for chance version of the mutual information.

\paragraph{Fowlkes-Mallows Index (FMI)} \citep{fowlkes1983method}. Supervised metric. Geometric mean between precision and recall on the edge classification task, as described for the ARI.

\paragraph{Modularity (Q)} \citep{newman2006modularity}. Unsupervised metric. Fraction of edges within clusters compared to that is some null model where edges are shuffled at random.

\paragraph{Normalized Standard Deviation (NSD)} Unsupervised metric. 1 minus normalized standard deviation in cluster size.

\subsection{Experimental setup}

All graphs are embedded in dimension 20, with different regularization parameters. To compare the impact of this parameter across different datasets, we use a relative regularization parameter $(w/n^2)\alpha$, where $w = 1_n^T A 1_n$ is the total weight of the graph. 

We use the K-Means algorithm with  to cluster the nodes in the embedding space. The parameter $K$ is set to  the ground-truth number of clusters (other experiments with different values of $K$ are reported in the Appendix).
We use the Scikit-learn \citep{scikit-learn} implementation of K-Means and the metrics, when available.
The spectral embedding and the modularity are computed with the Scikit-network package, see the documentation for more details\footnote{\url{https://scikit-network.readthedocs.io/}}. 

\subsection{Results}

We report the results  in Table  \ref{tab:wiki} for relative regularization parameter $\alpha = 0, 0.1, 1, 10$.
We see that the regularization generally improves performance, the optimal value of $\alpha$ depending on both the dataset and the score function. As suggested by Lemma \ref{lem:order_eig}, the optimal value of the regularization parameter should depend on the distribution of cluster sizes, on which we do not have any prior knowledge.



To test the impact of noise on the spectral embedding, we add isolated nodes with self loop to the graph and compare the clustering performance with and without regularization. The number of isolated nodes is given as a fraction of the initial number of nodes in the graph. Scores are computed only on the initial nodes.
The results are reported  
in Table \ref{tab:noise} for the Wikipedia for Schools dataset. We observe that,  in the absence of regularization, the scores drop even with only $1\%$ noise. The computed clustering is a trivial partition with all initial nodes in the same cluster. This means that the 20 first dimensions of the spectral embedding focus on the isolated nodes. On the other hand, the scores remain approximately constant in the regularized case, which suggests that  regularization makes the embedding  robust   to this type of noise.

\begin{table}[t]
    \centering
    \caption{Impact of regularization on clustering performance.}
  SBM
      \vspace{1mm}

  \begin{tabular}{l|M{1.1cm} M{1.1cm} M{1.1cm} M{1.1cm} M{1.1cm} M{1.1cm} M{1.1cm} M{1.1cm}}
\toprule
$\alpha$ &     H &     C &     V &   ARI &   AMI &   FMI &     Q &   NSD \\
\midrule
0   &  0.19 &  0.27 &  0.22 &  0.0 &  \textbf{0.01} &  \textbf{0.03} &  0.45 &  0.76 \\
0.1 &  0.33 &  0.35 &  0.34 &  0.0 &  \textbf{0.01} &  0.01 &  \textbf{0.52} &  0.91 \\
1   &  \textbf{0.36} &  \textbf{0.37} &  \textbf{0.36} &  0.0 &  \textbf{0.01} &  0.01 &  0.50 &  \textbf{0.92} \\
10  &  0.28 &  0.34 &  0.30 &  0.0 &  0.00 &  0.02 &  0.36 &  0.78 \\
\bottomrule
\end{tabular}
        \vspace{2mm}

NG        
    \vspace{1mm}

    \begin{tabular}{l|M{1.1cm} M{1.1cm} M{1.1cm} M{1.1cm} M{1.1cm} M{1.1cm} M{1.1cm} M{1.1cm}}
\toprule
$\alpha$ &     H &     C &     V &   ARI &   AMI &   FMI &     Q &   NSD \\
\midrule
0 &  0.40 &  \textbf{0.70} &  0.51 &  0.19 &  0.50 &  0.34 &  \textbf{0.21} &  0.55 \\
0.1 & 0.44 &  \textbf{0.70} &  \textbf{0.54} &  \textbf{0.22} &  \textbf{0.54} &  \textbf{0.35} &  \textbf{0.21} &  0.59 \\
1 & \textbf{0.46} &  0.67 &  \textbf{0.54} &  0.20 &  \textbf{0.54} &  0.33 &  0.20 &  \textbf{0.60} \\
10 & 0.37 &  0.55 &  0.45 &  0.13 &  0.44 &  0.26 &  0.17 &  0.56 \\
\bottomrule
\end{tabular}

        \vspace{2mm}

WS
    \vspace{1mm}

    \begin{tabular}{l|M{1.1cm} M{1.1cm} M{1.1cm} M{1.1cm} M{1.1cm} M{1.1cm} M{1.1cm} M{1.1cm}}
\toprule
$\alpha$ &     H &     C &     V &   ARI &   AMI &   FMI &     Q &   NSD \\
\midrule
0 &  0.23 &  \textbf{0.29} &  0.25 &  0.05 &  0.25 &  \textbf{0.26} &  0.25 &  0.49 \\
0.1 &  \textbf{0.26} &  \textbf{0.29} &  \textbf{0.28} &  \textbf{0.10} &  \textbf{0.27} &  \textbf{0.26} &  0.29 &  0.61 \\
1 &  0.23 &  0.24 &  0.23 &  0.04 &  0.23 &  0.20 &  \textbf{0.30} &  \textbf{0.65} \\
10 &  0.19 &  0.22 &  0.20 & 0.00 &  0.19 &  0.20 &  0.23 &  0.53 \\
\bottomrule
\end{tabular}

    \label{tab:wiki}
\end{table}

\begin{table}[h]
    \centering
    \caption{Impact of noise on  clustering performance (WS dataset).}
    $
    \alpha = 0
    $
\vspace{5mm}
    \begin{tabular}{l|M{1.1cm} M{1.1cm} M{1.1cm} M{1.1cm} M{1.1cm} M{1.1cm} M{1.1cm} M{1.1cm}}
\toprule
noise &     H &     C &     V &   ARI &   AMI &   FMI &     Q &   std \\
\midrule
0 $\%$ &  0.23 &  0.29 &  0.25 &  0.05 &  0.25 &  0.26 &  0.25 &  0.49 \\
1 $\%$ &  0.00 &  0.49 &  0.00 &  0.00 &  0.00 &  0.39 &  0.00 &  0 \\
5 $\%$ &  0.00 &  0.49 &  0.00 &  0.00 &  0.00 &  0.39 &  0.00 &  0 \\
10 $\%$ &  0.00 &  0.49 &  0.00 &  0.00 &  0.00 &  0.39 &  0.00 &  0 \\
\bottomrule
\end{tabular}
$
    \alpha = 1
    $
\vspace{5mm}
\begin{tabular}{l|M{1.1cm} M{1.1cm} M{1.1cm} M{1.1cm} M{1.1cm} M{1.1cm} M{1.1cm} M{1.1cm}}
\toprule
noise &     H &     C &     V &   ARI &   AMI &  FMI &    Q &   std \\
\midrule
0 $\%$ &  0.23 &  0.24 &  0.23 &  0.04 &  0.23 &  0.2 &  0.3 &  0.65 \\
1 $\%$ &  0.24 &  0.24 &  0.24 &  0.04 &  0.23 &  0.2 &  0.3 &  0.66 \\
5 $\%$ &  0.23 &  0.23 &  0.23 &  0.05 &  0.22 &  0.2 &  0.3 &  0.67 \\
10 $\%$ &  0.24 &  0.23 &  0.23 &  0.05 &  0.23 &  0.2 &  0.3 &  0.67 \\
\bottomrule
\end{tabular}

    \label{tab:noise}
\end{table}

\section{Conclusion and Perspectives}
\label{sec:conc}

In this paper, we have provided a simple explanation for the well-known benefits of regularization on spectral embedding. Specifically, regularization forces the  embedding to focus on the largest clusters, making the embedding more robust to noise. This result was obtained through the explicit characterization of the embedding for a simple block model, and extended to bipartite graphs.

An interesting perspective of our work is the extension to {\it stochastic} block models, using for instance the concentration results proved in \citep{lei2015consistency, le2017concentration}. Another problem of interest is the impact of regularization on other downstream tasks, like link prediction. Finally, we would like to further explore the impact of the regularization parameter, exploiting the theoretical results presented in this paper.

\newpage
\bibliography{iclr2020_conference}
\bibliographystyle{iclr2020_conference}

\appendix

\section*{Appendix}

We provide of proof of Theorem \ref{theo:svd} as well as a complete set of experimental results.

\section{Regularization of Bipartite Graphs}

The proof of Theorem \ref{theo:svd}  follows the same workflow as that of Theorem \ref{theo:main}. Let $Z_1 \in \mathbb{R}^{n \times K}$ and $Z_2 \in \mathbb{R}^{m \times K}$ be the left and right membership matrices for the block matrix $B \in \mathbb{R}^{n \times m}$. The aggregated matrix is $\bar{B} = Z_1^TBZ_2 \in \mathbb{R}^{K \times K}$. The diagonal matrices of block sizes are $W_1 = Z_1^TZ_1$ and $W_2 = Z_2^TZ_2$.
We have the equivalent of Proposition \ref{prop:aggregate}:

\begin{prop}
\label{prop:bip_aggregate}
Let $x_1,x_2$ be a solution to the generalized singular value  problem:
\begin{equation*}
\left\{
    \begin{array}{l}
    Bx_2 = \sigma D_1 x_1 \\
    B^Tx_1 = \sigma D_2 x_2 
    \end{array}\right.
\end{equation*}
Then either $Z_1^T x_1 = Z_2^T x_2 = 0$ and  $\sigma = 0$ or $x_1 = Z_1 y_1$ and $x_2 = Z_2y_2$ where 
 $y_1, y_2$ is a solution to the generalized singular value  problem:
\begin{equation*}
\left\{
    \begin{array}{l}
    \Bar{B}y_2 =  \sigma\Bar{D}_1y_1,\\
    \Bar{B}^Ty_1 =  \sigma\Bar{D}_2 y_2.
    \end{array}\right.
\end{equation*}
\end{prop}

\begin{proof}
Since the rank of $B$  is equal to $K$,  there are $n-K$  pairs of singular vectors $(x_1, x_2)$ associated with the singular values $0$,  each satisfying $Z_1^Tx_1 = 0$ and $Z_2^Tx_2 = 0$. By orthogonality, the other pairs of singular vectors satisfy $x_1 = Z_1y_1$ and $x_2 = Z_2y_2$ for some vectors $y_1, y_2 \in \mathbb{R}^K$. By replacing these in the original generalized singular value problem, 
we get that $(y_1, y_2)$ is a solution to the  generalized  singular value problem for the aggregate graph.
\end{proof}

In the following, we focus on the  block model described in Section \ref{sec:bipartite}, where  $B = Z_1Z_2^T$. 

{\it Proof of Lemma \ref{lem:bip1}.}
The generalized eigenvalue problem (\ref{eq:eigreg}) associated with the regularized matrix $A_\alpha$ is equivalent to the generalized SVD of the regularized biadjacency matrix $B_\alpha$:
\begin{equation*}
\left\{
    \begin{array}{l}
    B_\alpha x_2 = \sigma D_{\alpha,1} x_1  \\
    B_\alpha^Tx_1 = \sigma D_{\alpha, 2} x_2,
    \end{array}\right.
\end{equation*}
with  $\sigma = 1 - \lambda$.

In view of Proposition \ref{prop:bip_aggregate}, the singular value $\sigma = 0$ has multiplicity $n - K$, meaning that the eigenvalue $\lambda = 1$ has multiplicity $n-K$. Since the graph is connected, the eigenvalue 0 has multiplicity 1.
The proof then follows from the observation that if $(x_1,x_2)$ is a pair of  singular vectors for the  singular value $\sigma$, then the vectors $x=(x_1,\pm x_2)^T$ are eigenvectors for the eigenvalues $1-\sigma, 1+\sigma$. 

{\it Proof of Lemma \ref{lem:bip2}.}
By Proposition \ref{prop:bip_aggregate}, we can focus on the generalized singular value problem for the aggregate graph:
\begin{equation*}
\left\{
    \begin{array}{l}
    \bar B_\alpha y_2 = \sigma \bar D_{\alpha,1} y_1  \\
    \bar B_\alpha^Ty_1 = \sigma \bar D_{\alpha, 2} y_2,
    \end{array}\right.
\end{equation*}
Since
$$
\Bar{B}_{\alpha} = W_1(I_K + \alpha J_K)W_2,
$$
and
\begin{equation*}
\left\{
    \begin{array}{l}
    \Bar{D}_{\alpha, 1} = W_1(W_2 + \alpha n I),\\
    \Bar{D}_{\alpha, 2} = W_2(W_1 + \alpha m I),
    \end{array}\right.
\end{equation*}
we have:
\begin{equation*}
\left\{
    \begin{array}{l}
    W_1(I_K + \alpha J_K)W_2y_2 = W_1(W_2 + \alpha n I)y_1 \sigma,\\
    W_2(I_K + \alpha J_K)W_1y_1 = W_2(W_1 + \alpha m I) y_2 \sigma.
    \end{array}\right.
\end{equation*}

Observing that $J_KW_1y_1 \propto 1_K$ and $J_KW_2y_2 \propto 1_K$, we get:
\begin{equation*}
\left\{
    \begin{array}{l}
    (W_2 + \alpha m I_K)y_1\sigma - W_2y_2 \propto 1_K,\\
    (W_1 + \alpha n I_K)y_2\sigma - W_1y_1 \propto 1_K.
    \end{array}\right.
\end{equation*}

As two diagonal matrices commute, we obtain:
\begin{equation*}
\left\{
    \begin{array}{l}
    (W_1 + \alpha n I_K)(W_2 + \alpha m I_K)y_1\sigma - W_1W_2y_1 = \big(\eta_1(W_1 + \alpha n I_K) + \eta_2 W_2\big)1_K,\\
    (W_1 + \alpha n I_K)(W_2 + \alpha m I_K)y_2\sigma - W_1W_2y_2 = \big(\eta_1W_1 + \eta_2(W_2 + \alpha m I_K)\big)1_K,
    \end{array}\right.
\end{equation*}
for some constants $\eta_1, \eta_2$, and 
\begin{equation*}
\left\{
    \begin{array}{l}
    y_{1,j} = \dfrac{\eta_1(n_j + \alpha n) + \eta_2m_j}{(n_j + \alpha n)(m_j + \alpha m)\sigma - n_jm_j} ,\\
    y_{2,j} = \dfrac{\eta_1n_j + \eta_2(m_j + \alpha m)}{(n_j + \alpha n)(m_j + \alpha m)\sigma - n_jm_j}.
    \end{array}\right.
\end{equation*}
Letting $s_1 = -{\rm sign}(\eta_1(n_j + \alpha n) + \eta_2m_j)$ and $s_2 = - {\rm sign}(\eta_1n_j + \eta_2(m_j + \alpha m))$, we get:
$$
{\rm sign}(y_{1,j}) = s_1   \quad \Longleftrightarrow \quad {\rm sign}(y_{2,j}) = s_2   \quad \Longleftrightarrow \quad
 \dfrac{n_jm_j}{(n_j + \alpha n)(m_j + \alpha m)} \ge \sigma = 1 - \lambda,
$$
and the result follows from the fact that $x_1 = Z_1 y_1$ and $x_2 = Z_2 y_2$.

{\it Proof of Lemma \ref{lem:bip3}.}
The proof is the same as that of Lemma \ref{lem:order_eig}, where the threshold values follow from Lemma \ref{lem:bip2}:
\begin{equation*}
    \mu_j = 1 - \dfrac{n_jm_j}{(n_j + \alpha n)(m_j + \alpha m)}.
\end{equation*}

{\it Proof of Theorem \ref{theo:svd}.}
Let $x$ be the $j$-th column of the matrix $X$, for some $j\in \{2,\ldots,k\}$. In view of Lemma \ref{lem:bip3}, this is the eigenvector associated with eigenvalue $\lambda_j \in (\mu_{j-1},\mu_j)$.
In view of Lemma \ref{lem:bip1}, all entries of $x$ corresponding to blocks of size  $n_1,n_2\ldots,n_{j-1}$ have the same sign, the other having the opposite sign.

\section{Experimental Results}
In this section, we present more extensive experimental results.

Tables \ref{tab:2} and \ref{tab:half} present results for the same experiment as in Table \ref{tab:wiki} but for different values of $K$, namely $K=2$ (bisection of the graph) and $K = K_{\rm truth}/2$ (half of the ground-truth value). As for $K=K_{\rm true}$, regularization generally improves clustering performance. However, the optimal value of $\alpha$ remains both dataset dependent and metric dependent.
Note that, for the NG and WS datasets, the clustering remains trivial in the case $K=2$, one cluster containing all the nodes, until a certain amount of regularization.

\begin{table}[t]
    \centering
    \caption{Impact of regularization on clustering performance. $K = 2$.}
  SBM
      \vspace{1mm}

  \begin{tabular}{l|M{1.1cm} M{1.1cm} M{1.1cm} M{1.1cm} M{1.1cm} M{1.1cm} M{1.1cm} M{1.1cm}}
\toprule
$\alpha$ &     H &     C &     V &   ARI &   AMI &   FMI &     Q &   NSD \\
\midrule
0   &  0.00 &  0.43 &  0.00 &  0.0 &  0.0 &  \textbf{0.10} &  0.00 &  0.01 \\
0.1 &  0.00 &  \textbf{0.47} &  0.00 &  0.0 &  0.0 &  \textbf{0.10} &  0.00 &  0.00 \\
1   &  \textbf{0.01} &  0.04 &  \textbf{0.01} &  0.0 &  0.0 &  0.07 &  \textbf{0.34} &  \textbf{0.83} \\
10  &  \textbf{0.01} &  0.09 &  \textbf{0.01} &  0.0 &  0.0 &  0.09 &  0.13 &  0.22 \\
\bottomrule
\end{tabular}
        \vspace{2mm}

NG        
    \vspace{1mm}

    \begin{tabular}{l|M{1.1cm} M{1.1cm} M{1.1cm} M{1.1cm} M{1.1cm} M{1.1cm} M{1.1cm} M{1.1cm}}
\toprule
$\alpha$ &     H &     C &     V &   ARI &   AMI &   FMI &     Q &   NSD \\
\midrule
0   &  0.00 &  0.36 &  0.00 &  0.00 &  0.00 &  0.23 &  0.00 &  0.00 \\
0.1 &  0.00 &  0.36 &  0.00 &  0.00 &  0.00 &  0.23 &  0.00 &  0.00 \\
1   &  \textbf{0.15} &  \textbf{0.72} &  \textbf{0.25} &  \textbf{0.06} &  \textbf{0.25} &  \textbf{0.28} &  \textbf{0.16} &  \textbf{0.63} \\
10  &  0.12 &  0.61 &  0.20 &  0.04 &  0.20 &  0.26 &  0.13 &  0.51 \\
\bottomrule
\end{tabular}

        \vspace{2mm}

WS
    \vspace{1mm}

    \begin{tabular}{l|M{1.1cm} M{1.1cm} M{1.1cm} M{1.1cm} M{1.1cm} M{1.1cm} M{1.1cm} M{1.1cm}}
\toprule
$\alpha$ &     H &     C &     V &   ARI &   AMI &   FMI &     Q &   NSD \\
\midrule
0   &  0.00 &  \textbf{0.49} &  0.00 &  0.00 &  0.00 &  \textbf{0.39} &  0.00 &  0.00 \\
0.1 &  \textbf{0.07} &  0.42 &  \textbf{0.13} &  0.00 &  \textbf{0.12} &  0.34 &  0.09 &  \textbf{0.26} \\
1   &  0.03 &  0.27 &  0.05 & -0.01 &  0.05 &  0.35 &  0.09 &  0.13 \\
10  &  0.02 &  0.16 &  0.04 & -0.02 &  0.03 &  0.34 &  \textbf{0.10} &  0.16 \\
\bottomrule
\end{tabular}

    \label{tab:2}
\end{table}

\begin{table}[t]
    \centering
    \caption{Impact of regularization on clustering performance. $K = K_{\rm true} /2$.}
  SBM
      \vspace{1mm}

  \begin{tabular}{l|M{1.1cm} M{1.1cm} M{1.1cm} M{1.1cm} M{1.1cm} M{1.1cm} M{1.1cm} M{1.1cm}}
\toprule
$\alpha$ &     H &     C &     V &   ARI &   AMI &   FMI &     Q &   NSD \\
\midrule
0   &  0.08 &  0.21 &  0.11 &  0.0 &  0.00 &  \textbf{0.05} &  0.41 &  0.47 \\
0.1 &  0.20 &  0.27 &  0.23 &  0.0 &  \textbf{0.01} &  0.02 &  \textbf{0.55} &  0.84 \\
1   &  \textbf{0.24} &  \textbf{0.29} &  \textbf{0.26} &  0.0 &  0.00 &  0.02 &  0.54 &  \textbf{0.90} \\
10  &  0.19 &  0.28 &  0.23 &  0.0 &  0.00 &  0.03 &  0.40 &  0.70 \\
\bottomrule
\end{tabular}
        \vspace{2mm}

NG        
    \vspace{1mm}

    \begin{tabular}{l|M{1.1cm} M{1.1cm} M{1.1cm} M{1.1cm} M{1.1cm} M{1.1cm} M{1.1cm} M{1.1cm}}
\toprule
$\alpha$ &     H &     C &     V &   ARI &   AMI &   FMI &     Q &   NSD \\
\midrule
0   &  0.27 &  \textbf{0.76} &  0.40 &  0.11 &  0.39 &  0.31 &  0.20 &  0.41 \\
0.1 &  0.28 &  0.73 &  0.41 &  0.11 &  0.40 &  0.30 &  0.18 &  0.43 \\
1   &  \textbf{0.38} &  0.72 &  \textbf{0.50} &  \textbf{0.18} &  \textbf{0.50} &  \textbf{0.34} &  \textbf{0.21} &  \textbf{0.57} \\
10  &  0.31 &  0.62 &  0.42 &  0.11 &  0.42 &  0.27 &  0.17 &  0.51 \\
\bottomrule
\end{tabular}

        \vspace{2mm}

WS
    \vspace{1mm}

    \begin{tabular}{l|M{1.1cm} M{1.1cm} M{1.1cm} M{1.1cm} M{1.1cm} M{1.1cm} M{1.1cm} M{1.1cm}}
\toprule
$\alpha$ &     H &     C &     V &   ARI &   AMI &   FMI &     Q &   NSD \\
\midrule
0 &  0.23 &  \textbf{0.29} &  0.25 &  0.05 &  0.25 &  \textbf{0.26} &  0.25 &  0.49 \\
0.1 &  \textbf{0.26} &  \textbf{0.29} &  \textbf{0.28} &  \textbf{0.10} &  \textbf{0.27} &  \textbf{0.26} &  0.29 &  0.61 \\
1 &  0.23 &  0.24 &  0.23 &  0.04 &  0.23 &  0.20 &  \textbf{0.30} &  \textbf{0.65} \\
10 &  0.19 &  0.22 &  0.20 & -0.00 &  0.19 &  0.20 &  0.23 &  0.53 \\
\bottomrule
\end{tabular}

    \label{tab:half}
\end{table}

Table \ref{tab:uvb} presents the different scores for both types of regularization on the NG dataset. As we can see, preserving the bipartite structure of the graph leads to slightly better performance.

Finally, Table \ref{tab:noise_ng} shows the impact of regularization in the presence of noise for the NG dataset.  The conclusions are similar as for the WS dataset: regularization makes the spectral embedding much more robust to noise.

\begin{table}
    \centering
    \caption{Regularization of  the adjacency  vs.  biadjacency matrix on the NG dataset ($\alpha=1$).}
    \vspace{1mm}
    
$K = K_{\rm true}/2$
    \vspace{2mm}
    
\begin{tabular}{l|M{1.1cm} M{1.1cm} M{1.1cm} M{1.1cm} M{1.1cm} M{1.1cm} M{1.1cm} M{1.1cm}}
\toprule
{} &     H &     C &     V &   ARI &   AMI &   FMI &     Q &   std \\
\midrule
Adj. &  0.38 &  0.72 &  0.50 &  0.18 &  0.50 &  0.34 &  0.21 &  0.57 \\
Biadj. &  \textbf{0.41} &  0.72 &  \textbf{0.52} &  \textbf{0.19} &  \textbf{0.52} &  \textbf{0.35} &  0.21 &  \textbf{0.61} \\
\bottomrule
\end{tabular}

    \vspace{2mm}
    
$K = K_{\rm true}$
    \vspace{1mm}
    
\begin{tabular}{l|M{1.1cm} M{1.1cm} M{1.1cm} M{1.1cm} M{1.1cm} M{1.1cm} M{1.1cm} M{1.1cm}}
\toprule
{} &     H &     C &     V &   ARI &   AMI &   FMI &    Q &   std \\
\midrule
Adj. &  0.46 &  0.67 &  0.54 &  0.20 &  0.54 &  0.33 &  0.2 &  0.60 \\
Biadj. &  \textbf{0.47} &  \textbf{0.68} &  \textbf{0.56} &  \textbf{0.21} &  \textbf{0.55} &  \textbf{0.34} &  0.2 &  \textbf{0.61} \\
\bottomrule
\end{tabular}

    \label{tab:uvb}
\end{table}

\begin{table}[h]
    \centering
    \caption{Impact of noise on  clustering performance (NG dataset).}
    $
    \alpha = 0
    $
\vspace{5mm}
    \begin{tabular}{l|M{1.1cm} M{1.1cm} M{1.1cm} M{1.1cm} M{1.1cm} M{1.1cm} M{1.1cm} M{1.1cm}}
\toprule
noise &     H &     C &     V &   ARI &   AMI &   FMI &     Q &   std \\
\midrule
0 $\%$ &   0.40 &  {0.70} &  0.51 &  0.19 &  0.50 &  0.34 &  {0.21} &  0.55 \\
1 $\%$ &  0.00 &  1.00 & 0.00 &  0.00 & 0.00 &  0.23 & 0.00 &   0 \\
5 $\%$ &  0.14 &  0.65 &  0.23 &  0.06 &  0.23 &  0.27 &  0.13 &  0.30 \\
10 $\%$ &  0.00 &  0.36 &  0.01 & 0.00 &  0.00 &  0.23 &  0.00 &  0.00 \\
\bottomrule
\end{tabular}
$
    \alpha = 1
    $
\vspace{5mm}
\begin{tabular}{l|M{1.1cm} M{1.1cm} M{1.1cm} M{1.1cm} M{1.1cm} M{1.1cm} M{1.1cm} M{1.1cm}}
\toprule
noise &     H &     C &     V &   ARI &   AMI &  FMI &    Q &   std \\
\midrule
0 $\%$ &  {0.46} &  0.67 &  {0.54} &  0.20 &  {0.54} &  0.33 &  0.2 &  {0.60} \\
1 $\%$ &  0.48 &  0.66 &  0.56 &  0.21 &  0.56 &  0.33 &  0.2 &  0.64 \\
5 $\%$ &  0.49 &  0.66 &  0.56 &  0.23 &  0.56 &  0.34 &  0.2 &  0.66 \\
10 $\%$ &  0.45 &  0.66 &  0.54 &  0.20 &  0.54 &  0.33 &  0.2 &  0.59 \\
\bottomrule
\end{tabular}

    \label{tab:noise_ng}
\end{table}

\end{document}

%% file: math_commands.tex
\usepackage{amsmath,amsfonts,bm}

\def\eqref#1{equation~\ref{#1}}

\def\1{\bm{1}}

\DeclareMathAlphabet{\mathsfit}{\encodingdefault}{\sfdefault}{m}{sl}
\SetMathAlphabet{\mathsfit}{bold}{\encodingdefault}{\sfdefault}{bx}{n}













%% file: ICLR2020.bbl
\begin{thebibliography}{23}
\providecommand{\natexlab}[1]{#1}
\providecommand{\url}[1]{\texttt{#1}}
\expandafter\ifx\csname urlstyle\endcsname\relax
  \providecommand{\doi}[1]{doi: #1}\else
  \providecommand{\doi}{doi: \begingroup \urlstyle{rm}\Url}\fi

\bibitem[Amini et~al.(2013)Amini, Chen, Bickel, Levina,
  et~al.]{amini2013pseudo}
Arash~A Amini, Aiyou Chen, Peter~J Bickel, Elizaveta Levina, et~al.
\newblock Pseudo-likelihood methods for community detection in large sparse
  networks.
\newblock \emph{The Annals of Statistics}, 41\penalty0 (4):\penalty0
  2097--2122, 2013.

\bibitem[Belkin \& Niyogi(2002)Belkin and Niyogi]{belkin2002laplacian}
Mikhail Belkin and Partha Niyogi.
\newblock Laplacian eigenmaps and spectral techniques for embedding and
  clustering.
\newblock In \emph{Advances in neural information processing systems}, pp.\
  585--591, 2002.

\bibitem[Bonald et~al.(2018)Bonald, Hollocou, and Lelarge]{bonald2018weighted}
Thomas Bonald, Alexandre Hollocou, and Marc Lelarge.
\newblock Weighted spectral embedding of graphs.
\newblock In \emph{2018 56th Annual Allerton Conference on Communication,
  Control, and Computing (Allerton)}, pp.\  494--501. IEEE, 2018.

\bibitem[Chaudhuri et~al.(2012)Chaudhuri, Chung, and
  Tsiatas]{chaudhuri2012spectral}
Kamalika Chaudhuri, Fan Chung, and Alexander Tsiatas.
\newblock Spectral clustering of graphs with general degrees in the extended
  planted partition model.
\newblock In \emph{Conference on Learning Theory}, pp.\  35--1, 2012.

\bibitem[Chung(1997)]{chung}
Fan~RK Chung.
\newblock \emph{Spectral graph theory}.
\newblock American Mathematical Soc., 1997.

\bibitem[Fowlkes \& Mallows(1983)Fowlkes and Mallows]{fowlkes1983method}
Edward~B Fowlkes and Colin~L Mallows.
\newblock A method for comparing two hierarchical clusterings.
\newblock \emph{Journal of the American statistical association}, 78\penalty0
  (383):\penalty0 553--569, 1983.

\bibitem[Haruechaiyasak \& Damrongrat(2008)Haruechaiyasak and
  Damrongrat]{haruechaiyasak2008article}
Choochart Haruechaiyasak and Chaianun Damrongrat.
\newblock Article recommendation based on a topic model for wikipedia selection
  for schools.
\newblock In \emph{International Conference on Asian Digital Libraries}, pp.\
  339--342. Springer, 2008.

\bibitem[Holland et~al.(1983)Holland, Laskey, and
  Leinhardt]{holland1983stochastic}
Paul~W Holland, Kathryn~Blackmond Laskey, and Samuel Leinhardt.
\newblock Stochastic blockmodels: First steps.
\newblock \emph{Social networks}, 5\penalty0 (2):\penalty0 109--137, 1983.

\bibitem[Hubert \& Arabie(1985)Hubert and Arabie]{hubert1985comparing}
Lawrence Hubert and Phipps Arabie.
\newblock Comparing partitions.
\newblock \emph{Journal of classification}, 2\penalty0 (1):\penalty0 193--218,
  1985.

\bibitem[Joseph et~al.(2016)Joseph, Yu, et~al.]{joseph2016impact}
Antony Joseph, Bin Yu, et~al.
\newblock Impact of regularization on spectral clustering.
\newblock \emph{The Annals of Statistics}, 44\penalty0 (4):\penalty0
  1765--1791, 2016.

\bibitem[Lara(2019)]{mlg2019_1}
Nathan~De Lara.
\newblock The sparse + low rank trick for matrix factorization-based graph
  algorithms.
\newblock In \emph{Proceedings of the 15th International Workshop on Mining and
  Learning with Graphs (MLG)}, 2019.

\bibitem[Le et~al.(2017)Le, Levina, and Vershynin]{le2017concentration}
Can~M Le, Elizaveta Levina, and Roman Vershynin.
\newblock Concentration and regularization of random graphs.
\newblock \emph{Random Structures \& Algorithms}, 51\penalty0 (3):\penalty0
  538--561, 2017.

\bibitem[Lei et~al.(2015)Lei, Rinaldo, et~al.]{lei2015consistency}
Jing Lei, Alessandro Rinaldo, et~al.
\newblock Consistency of spectral clustering in stochastic block models.
\newblock \emph{The Annals of Statistics}, 43\penalty0 (1):\penalty0 215--237,
  2015.

\bibitem[Luxburg(2007)]{luxburg07}
Ulrike Luxburg.
\newblock A tutorial on spectral clustering.
\newblock \emph{Statistics and Computing}, 17\penalty0 (4):\penalty0 395--416,
  December 2007.
\newblock ISSN 0960-3174.
\newblock \doi{10.1007/s11222-007-9033-z}.
\newblock URL \url{http://dx.doi.org/10.1007/s11222-007-9033-z}.

\bibitem[Newman(2006)]{newman2006modularity}
Mark~EJ Newman.
\newblock Modularity and community structure in networks.
\newblock \emph{Proceedings of the national academy of sciences}, 103\penalty0
  (23):\penalty0 8577--8582, 2006.

\bibitem[Ng et~al.(2002)Ng, Jordan, and Weiss]{ng2002spectral}
Andrew~Y Ng, Michael~I Jordan, and Yair Weiss.
\newblock On spectral clustering: Analysis and an algorithm.
\newblock In \emph{Advances in neural information processing systems}, pp.\
  849--856, 2002.

\bibitem[Pedregosa et~al.(2011)Pedregosa, Varoquaux, Gramfort, Michel, Thirion,
  Grisel, Blondel, Prettenhofer, Weiss, Dubourg, Vanderplas, Passos,
  Cournapeau, Brucher, Perrot, and Duchesnay]{scikit-learn}
F.~Pedregosa, G.~Varoquaux, A.~Gramfort, V.~Michel, B.~Thirion, O.~Grisel,
  M.~Blondel, P.~Prettenhofer, R.~Weiss, V.~Dubourg, J.~Vanderplas, A.~Passos,
  D.~Cournapeau, M.~Brucher, M.~Perrot, and E.~Duchesnay.
\newblock Scikit-learn: Machine learning in {P}ython.
\newblock \emph{Journal of Machine Learning Research}, 12:\penalty0 2825--2830,
  2011.

\bibitem[Qin \& Rohe(2013)Qin and Rohe]{qin2013regularized}
Tai Qin and Karl Rohe.
\newblock Regularized spectral clustering under the degree-corrected stochastic
  blockmodel.
\newblock In \emph{Advances in Neural Information Processing Systems}, pp.\
  3120--3128, 2013.

\bibitem[Rosenberg \& Hirschberg(2007)Rosenberg and Hirschberg]{rosenberg2007v}
Andrew Rosenberg and Julia Hirschberg.
\newblock V-measure: A conditional entropy-based external cluster evaluation
  measure.
\newblock In \emph{Proceedings of the 2007 joint conference on empirical
  methods in natural language processing and computational natural language
  learning (EMNLP-CoNLL)}, pp.\  410--420, 2007.

\bibitem[Snell \& Doyle(2000)Snell and Doyle]{snell00}
P~Snell and Peter Doyle.
\newblock Random walks and electric networks.
\newblock \emph{Free Software Foundation}, 2000.

\bibitem[Spielman(2007)]{spielman2007spectral}
Daniel~A Spielman.
\newblock Spectral graph theory and its applications.
\newblock In \emph{Foundations of Computer Science, 2007. FOCS'07. 48th Annual
  IEEE Symposium on}, pp.\  29--38. IEEE, 2007.

\bibitem[Vinh et~al.(2010)Vinh, Epps, and Bailey]{vinh2010information}
Nguyen~Xuan Vinh, Julien Epps, and James Bailey.
\newblock Information theoretic measures for clusterings comparison: Variants,
  properties, normalization and correction for chance.
\newblock \emph{Journal of Machine Learning Research}, 11\penalty0
  (Oct):\penalty0 2837--2854, 2010.

\bibitem[Zhang \& Rohe(2018)Zhang and Rohe]{zhang2018understanding}
Yilin Zhang and Karl Rohe.
\newblock Understanding regularized spectral clustering via graph conductance.
\newblock In \emph{Advances in Neural Information Processing Systems}, pp.\
  10631--10640, 2018.

\end{thebibliography}
